\documentclass[conference]{IEEEtran}
\usepackage{times}

\usepackage{multicol}
\usepackage[bookmarks=true]{hyperref}

\usepackage{amsmath,amssymb,amsfonts}
\usepackage{algorithmic}
\usepackage{graphicx}
\usepackage{textcomp}
\usepackage{setspace}
\usepackage{booktabs}

\usepackage{epsfig}
\usepackage{times} 
\usepackage{svg}
\usepackage[linesnumbered,ruled,vlined]{algorithm2e} 

\usepackage[nolist, nohyperlinks]{acronym}
\acrodef{GNC}[GNC]{Graduated-Nonconvexity}
\acrodef{WLS}[WLS]{Weighted Least Squares}
\acrodef{TLS}[TLS]{Truncated Least Squares}
\acrodef{SVD}[SVD]{Singular Value Decomposition}
\acrodef{VO}[VO]{Visual Odometry}
\acrodef{VIO}[VIO]{Visual Inertial Odometry}
\acrodef{SDF}[SDF]{Signed Distance Field}
\acrodef{SDFs}[SDFs]{Signed Distance Fields}
\acrodef{ESDF}[ESDF]{Euclidean Signed Distance Field}
\acrodef{CESDF}[C-ESDF]{Certified \ac{ESDF}}
\acrodef{TSDF}[TSDF]{Truncated Signed Distance Field}
\acrodef{RGBD}[RGBD]{RGB-Depth}
\acrodef{IMU}[IMU]{Inertial Measurement Unit}
\acrodef{EKF}[EKF]{Extended Kalman Filter}
\acrodef{VO}[VO]{Visual Odometry}
\acrodef{CVO}[C-VO]{Certified Visual Odometry}
\acrodef{ISS}[ISS]{Input-to-State}
\acrodef{SLAM}[SLAM]{Simultaneous Localization and Mapping}
\acrodef{SFC}[SFC]{Safe Flight Corridors}
\acrodef{FOV}[FoV]{Field of View}

\usepackage{bm}
\usepackage{breqn}
\usepackage{amsmath}
\usepackage{amssymb}
\usepackage{tabularx}

\usepackage{amsthm}

\newcommand{\naturals}{\mathbb{N}}

\newcommand{\reals}{\mathbb{R}}
\newcommand{\R}{\reals}
\newcommand{\Rnonneg}{\reals_{\geq 0}}
\newcommand{\Rplus}{\reals_{>0}}

\renewcommand{\S}{\mathbb{S}}
\newcommand{\SO}{\mathbb{SO}}
\newcommand{\SE}{\mathbb{SE}}

\newcommand{\Acal}{\mathcal{A}}
\newcommand{\Bcal}{\mathcal{B}}

\newcommand{\Fcal}{\mathcal{F}}

\newcommand{\Ical}{\mathcal{I}}

\newcommand{\Ocal}{\mathcal{O}}
\newcommand{\Pcal}{\mathcal{P}}

\newcommand{\eqn}[1]{\begin{align} #1 \end{align}}
\newcommand{\eqnN}[1]{\begin{align*} #1 \end{align*}}

\newcommand{\bmat}[1]{\begin{bmatrix}#1\end{bmatrix}}

\newcommand{\norm}[1]{\left\Vert #1 \right \Vert}
\newcommand{\fronorm}[1]{\left\Vert #1 \right \Vert_F}
\newcommand{\abs}[1]{\left | #1 \right |}

\newcommand{\tr}[1]{\operatorname{tr}\left(#1\right)}

\newcommand{\argmin}[1]{\underset{\substack{#1}}{\text{argmin}}}

\newcommand{\inframe}[1]{|^{#1}}
\newcommand{\rotation}[2]{R_{#1}^{#2}}
\newcommand{\estrotation}[2]{\hat{R}_{#1}^{#2}}
\newcommand{\translation}[2]{t_{#1}^{#2}}
\newcommand{\esttranslation}[2]{\hat{t}_{#1}^{#2}}
\newcommand{\rototranslation}[2]{(\rotation{#1}{#2}, \translation{#1}{#2})}
\newcommand{\estrototranslation}[2]{(\estrotation{#1}{#2}, \esttranslation{#1}{#2})}

\newcommand{\Brotation}[2]{R_{B_{#1}}^{B_{#2}}}
\newcommand{\estBrotation}[2]{\hat{R}_{B_{#1}}^{B_{#2}}}
\newcommand{\Btranslation}[2]{t_{B_{#1}}^{B_{#2}}}
\newcommand{\estBtranslation}[2]{\hat{t}_{B_{#1}}^{B_{#2}}}

\newcommand{\estBrototranslation}[2]{(\estBrotation{#1}{#2}, \estBtranslation{#1}{#2})}

\theoremstyle{plain}
\newtheorem{theorem}{Theorem}

\newtheorem{lemma}{Lemma}
\newtheorem{problem}{Problem}
\newtheorem{definition}{Definition}

\theoremstyle{definition}
\newtheorem{assumption}{Assumption}
\newtheorem{remark}{Remark}

\theoremstyle{remark}

\makeatletter
\let\NAT@parse\undefined
\makeatother
\usepackage{hyperref}
\hypersetup{
colorlinks, 
    linkcolor={red!50!black},
    citecolor={blue!50!black},
    urlcolor={blue!80!black}
}
\usepackage{cleveref}

\crefformat{problem}{Problem~#2#1#3}
\crefformat{assumption}{Assumption~#2#1#3}

\begin{document}

\title{Online and Certifiably Correct \\ Visual Odometry and Mapping}

\ifx\anonymized\undefined 
\author{Devansh R. Agrawal,$^{1*}$ Rajiv Govindjee,$^{1}$ Jiangbo Yu,$^{2}$ Anurekha Ravikumar,$^{3}$ and Dimitra Panagou$^{1,3}$ 
\thanks{The authors would like to acknowledge the support of the National Science Foundation through the Award No. 1942907. This work was partially supported by Amazon.}
\thanks{$^{1}$Department of Aerospace Engineering, University of Michigan, 1320 Beal Ave, Ann Arbor, MI 48109, USA.}%
\thanks{$^{2}$Department of Mechanical Engineering, University of Michigan, 2350 Hayward St, Ann Arbor, MI 48109, USA.}%
\thanks{$^{3}$Department of Robotics, University of Michigan, 2505 Hayward St, Ann Arbor, MI 48109, USA.}%
\thanks{$^{*}$Correspondence: {\tt\small devansh@umich.edu} }
}
\else
\author{Author Names Omitted for Anonymous Review. Paper-ID 256}
\thanks{$^{*}$[redacted]}
}
\fi

\maketitle

\begin{abstract}
This paper proposes two new algorithms for certified perception in safety-critical robotic applications. The first is a Certified Visual Odometry algorithm, which uses a RGBD camera with bounded sensor noise to construct a visual odometry estimate with provable error bounds. The second is a Certified Mapping algorithm which, using the same RGBD images, constructs a Signed Distance Field of the obstacle environment, always safely underestimating the distance to the nearest obstacle. This is required to avoid errors due to VO drift. The algorithms are demonstrated in hardware experiments, where we demonstrate both running online at 30FPS. The methods are also compared to state-of-the-art techniques for odometry and mapping.
\end{abstract}

\IEEEpeerreviewmaketitle

\section{Introduction}
\label{sec:introduction}

Accurate and reliable perception, state estimation, and mapping are critical components of a robotic system. When operating safety-critical robotic systems, the planners and controllers rely on the outputs of the perception module to determine whether a planned trajectory or control action is safe. Over the last decade, many methods have been developed to certify that a controller satisfies safety specifications prescribed a priori, e.g.~\cite{ames2016control, garg2023advances}. If the specification is to avoid obstacles, and these obstacles can only be sensed online, it can be more natural to address safety constraints using the planning module~\cite{lopez2017aggressive, tordesillas2019faster, agrawal2023gatekeeper}. However these approaches tend to assume perfect information from the perception system, an unrealistic assumption that can lead to safety violations. 

A perception module returns a pose estimate and a map representing the obstacles in an environment. The representations can take many forms, including an \ac{ESDF}~\cite{oleynikova2017voxblox, nvblox}, occupancy log-odds~\cite{hornung2013octomap} or NERFs~\cite{rosinol2023nerf}. These methods construct ``best-estimate'' maps, and do not quantify the error. Without error bounds, the planners/controllers are unable to guarantee safety.

In state estimation, there are some methods to determine error bounds. A Kalman Filter for example quantifies the state estimation uncertainty in terms of a covariance ellipsoid. There is also growing interest in certified perception techniques, i.e., algorithms that provably recover the globally optimal solution or in some cases those that can provide error bounds with respect to the optimal solution. For example, the algorithm in~\cite{rosen2019se} provably returns the global optimum to a pose-graph optimization problem, by reformulating it into a convex optimization problem. The algorithm in~\cite{marchi2022lidar} can determine the location of a robot in a convex 2D environment with error bounds. On the other hand, the accuracy of (uncertified) perception algorithms have been improving in recent years, and many experimental demonstrations show good performance in GPS-denied environments~\cite{scaramuzza2011visual, yu2021vins, chen2023direct, tian2022kimera}. To the best of our knowledge, no formal error analysis is available for these methods. 


Currently, robust safety-critical planners/controllers are designed to handle specific forms of uncertainty. These include additive dynamics disturbances that are either bounded or stochastic~\cite{kolathaya2018input, black2023safety, culbertson2023input}, or state estimation errors that are, again, bounded or stochastic~\cite{dean2021guaranteeing, agrawal2022safe}. Therefore our perception module must return estimates and error bounds compatible with planners and controllers.

This paper takes two steps towards the goal of certified perception-planning-control algorithms. First, we propose a new \ac{CVO} algorithm for which (under appropriate assumptions) we can establish a bounded odometry error, i.e. error bounds on the relative pose between successive camera frames. As these errors accumulate over time, the position estimate will deviate from the ground truth position of the robot, as is commonly observed in \ac{VO} systems \cite{scaramuzza2011visual, yu2021vins}. If using mapping frameworks like VoxBlox \cite{oleynikova2017voxblox}, this can lead to some occupied regions marked free in the map. Therefore, our second key contribution is a \ac{CESDF} algorithm. Building on the framework of \cite{nvblox}, we introduce a deflation step, ensuring that the \ac{CESDF} always underestimates the distance to the nearest obstacle in a body-fixed frame. To the best of the authors' knowledge, this is the first algorithm that can certify the \ac{VO} and \ac{ESDF} outputs of a perception stack. Finally, we have performed experiments to demonstrate that the proposed algorithms can run in real-time.

The remainder of the paper is structured as follows. In \Cref{section:problem_statement} we formalize the problem, and in \Cref{section:method_overview} we describe conceptually the key components of our method. The main theoretical contributions are detailed in Sections \ref{sec:certified-visual-odometry} and \ref{sec:certified-esdf}. These theoretical contributions are assembled into a complete algorithm in \Cref{section:algorithms}. Finally in \Cref{sec:experiments} we report experiments demonstrating the proposed methods.

\subsubsection*{Notation} 
\label{sec:notation}

$\naturals= \{0, 1, 2, ...\}$ is the set of natural numbers. $\R, \Rnonneg, \Rplus$ denote reals, non-negative reals, and positive reals. $\S^3$ is the set of unit quaternions. $\SO(3)$ is the special orthogonal group. $\SE(3) = \SO(3) \times \R^3$ is the special Euclidean group. $\norm{v}_p$ denotes the $p$-norm of a vector, and $\norm{v}$ denotes the 2-norm. For matrices $A \in \R^{n \times m}$, $\norm{A}$ denotes the induced 2-norm, and $\fronorm{A}$ denotes the Frobenius norm. All eigenvectors are assumed to be unit-norm. 
A unit quaternion $q \in \S^3$ is denoted $q=\bmat{q_1 & q_2 & q_3& q_4}^T$, where the scalar component is last. The inverse is $q^{-1} = \bmat{-q_1 & -q_2 & -q_3& q_4}^T$. The quaternion product is $q_c = q_a \circ q_b = \Omega_1(q_a) q_b = \Omega_2(q_b) q_a$,  where
{\eqnN{
\Omega_1(q) = \bmat{ 
q_4 & -q_3 & q_2 & q_1\\
q_3 & q_4 & -q_1 & q_2 \\
-q_2 & q_1 & q_4 & q_3 \\
-q_1 & -q_2 & -q_3 & q_4
}, \\
\Omega_2(q) = \bmat{ 
q_4 & q_3 & -q_2 & q_1\\
-q_3 & q_4 & q_1 & q_2 \\
q_2 & -q_1 & q_4 & q_3 \\
-q_1 & -q_2 & -q_3 & q_4
}.
}}
Every $q\in\S^3$ is associated with a rotation $R \in \SO(3)$, $\overline{R a} = q \circ \overline a \circ q^{-1}$,  where $\overline v = \bmat{v^T & 0 }^T \in \R^4$ for any $v \in \R^3$. 
Therefore, we have the useful properties  $\Omega_1(q^{-1}) = \Omega_1^T(q)$,  $\Omega_2(q^{-1}) = \Omega_2^T(q)$, and $q \circ \overline a \circ q^{-1} = \Omega_2^T(q) \Omega_1(q) \bar a$.

A reference frame is a set of three orthonormal basis vectors and an origin. When a point $p$ is expressed in reference frame $M$, it is denoted $p \inframe{M} \in \R^3$. Two frames $A, B$ are be related by a rototranslation $(R, t) \in \SE(3)$,  $p\inframe{B} = \rotation{A}{B} p \inframe{A} + \translation{A}{B}$.
This paper uses the inertial frame $I$, and a mapping frame $M$. The body frame $B_k$ denotes the body-fixed reference frame at timestep $k$, i.e., when the $k$-th RGBD image is produced.  See also \Cref{tab:notation} in the appendix.

\section{Problem Statement}
\label{section:problem_statement}

The goal in this paper is to identify a subset of $\R^3$ in the body-frame that is certifiably obstacle-free. We constrain ourselves to using an \ac{RGBD} camera without access to an \ac{IMU}.\footnote{Using an \ac{IMU} can improve the accuracy of localization algorithms and provide robustness against changes in lighting conditions~\cite{qin2018vins, mohta2018experiments}. However, certifying error bounds is challenging due to the IMU biases. Incorporating an IMU is considered future work.} 
The depth sensor generates a measured pointcloud of the obstacles within its \ac{FOV}. The position of each point in the pointcloud could have some measurement error. We assume this error is bounded with a known bound.\footnote{According to manufacturer specifications, each point of the pointcloud has an absolute position error of 2\% of the distance from the camera~\cite{realsense}. Empirical studies have suggested a quadratic relationship~\cite{nguyen2012modeling}.} Then, we have the following problem statement:

\begin{problem}
\label{problem:overall}
Consider a robotic system with an onboard \ac{RGBD} camera operating at a fixed frame rate in a static environment. Suppose the depth camera produces pointclouds with a bounded position error in each point. Identify a subset $\Fcal \subset \R^3$ that is guaranteed to be obstacle free. 
\end{problem}

We assume that if a point $p \in \R^3$ is occupied, and within the camera's \ac{FOV}, it will be detected as an obstacle. This is a common implicit assumption in the mapping literature. Note, an infrared depth camera often fails to detect transparent obstacles (e.g., windows and glass doors). Such issues are beyond the scope of this paper.

\section{Method Overview}
\label{section:method_overview}

Our solution decomposes the certified perception problem into two steps, the certified state estimation problem and the certified mapping problem, as depicted in~\Cref{fig:block_diagram}. Although separate, these modules have been designed to integrate together. Here we describe the modules conceptually, and explain our choices in \cref{section:why_esdfs}. 

\begin{figure}
    \centering
    \includegraphics[width=0.96\linewidth]{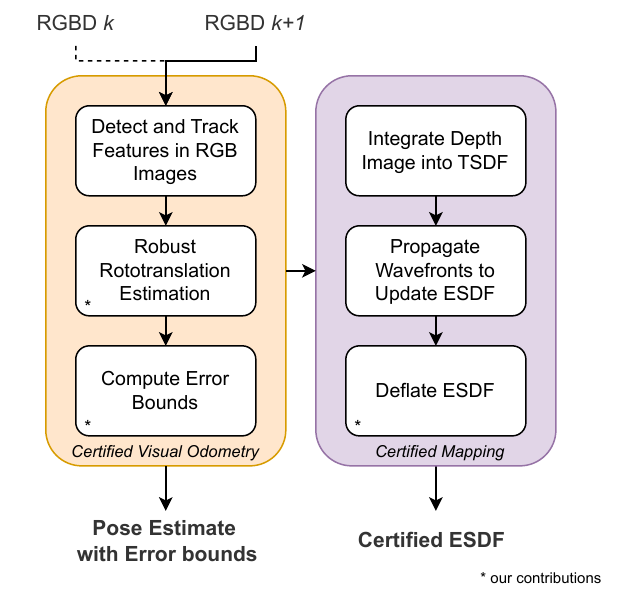}
    \caption{Block diagram describing our certified perception modules. Given successive \ac{RGBD} frames, we first use the visual odometry (orange) module to compute a relative pose estimate, and the associated error bounds. These are used in the mapping (purple) module to construct a 3D map of the obstacle geometry. Using the error bounds computed in \ac{CVO}, we can compute the \ac{ESDF} deflation that is sufficient to ensure correctness. }
    \label{fig:block_diagram}
\end{figure}

\subsection{C-VO}

First, we propose a new method for \ac{CVO}. By \emph{certified} we mean that our \ac{VO} algorithm produces an estimate of the odometry between successive timesteps, \emph{and} produces an upper-bound on the maximum error between the estimated odometry and the true odometry.

We adapt the robust pointcloud registration technique in~\cite{yang2020teaser} to estimate the rototranslation between successive timesteps. Using the bounded error of the sensed pointcloud, we derive an upperbound on the rototranslation estimation error. More precisely in \Cref{corrollary:rotation_bound} we determine a Frobenius norm bound on the rotation error, and in \Cref{lemma:translation_bound} we determine a two-norm bound on the translation error. 

\subsection{C-ESDF}

Second, we propose an algorithm to construct a \ac{CESDF} to represent the world. New depth images are integrated into into a voxelized representation of the world, where each voxel stores the distance to the nearest obstacle, an \ac{ESDF}. As the robot moves through the environment, the accumulated \ac{VO} drift implies that the mapping reference frame shifts relative to the inertial frame. Since the obstacles are not static in the mapping frame, the \ac{ESDF} can become incorrect. 

To account for this, we introduce a deflation step to~\cite{nvblox}. We derive a  recursive guarantee that if the \ac{CESDF} safely underestimates the distance to every obstacle at timestep $k$, after deflating the \ac{CESDF} it will again safely underestimate the \ac{SDF} at timestep $k+1$. 

We accomplish this by, at every frame, decrementing the \ac{CESDF} at each voxel based on the \ac{CVO} error bound. The deflation amount is different at each voxel, and is computed from the error bounds produced by the \ac{CVO} algorithm as described in \Cref{sec:certified-esdf}. This effectively causes the system to `forget' parts of the map not been observed recently.  The result is a map that is always guaranteed to be safe for use in path planning and control.

\subsection{Why Pointcloud Registration and ESDFs?}
\label{section:why_esdfs}

The proposed approach has some key properties that enables certification guarantees. Here we compare our strategy to possible alternatives.

\emph{Compared to Epipolar \ac{VO}~\cite{cvivsic2022soft2}:} A common technique in \ac{VO} is to use Epipolar constraints to determine the rototranslation between frames. However due to the geometric properties used in deriving the algorithm, it is not clear how to compute error bounds for this algorithm.

\emph{Compared to Octomap~\cite{hornung2013octomap}:} An Octomap stores the occupancy log-odds. However since odometry error is bounded by a norm ball, it is unclear how to update the log-odds.

\emph{Compared to \ac{SFC}~\cite{liu2017planning}:} Again, based on a norm-ball error bound, the algebra necessary to update the SFC polyhedron is not clear.\footnote{Consider a SFC $F = \{x \in \R^3 : Ax \leq b\}$. If the rototranslation from frame $k$ to $k+1$ is $(R, t) \in \SE(3)$, the new safe region is $F' = \{ x \in \R^3 : A R^T x \leq b + A R^T t \}$. However if there is uncertainty in $(R, t)$, to the best of our knowledge, there is no analytic method to compute $F'$.} 

\emph{Compared to inflating the obstacles:} One could maintain a list of observed obstacles, and inflate their radius based on the accumulated \ac{VO} error. With this approach however, the spheres could grow to occupy all of $\R^3$, preventing any new region from being certified safe. Instead, we choose to deflate safe regions. Although this can lead to a situation where no subset of $\R^3$ is certifiably safe, as new depth images are received, new regions can be added to the certified-free set. 

Thus, we choose to use the proposed \ac{CVO} and \ac{CESDF} algorithms. Naturally, from the \ac{CESDF} a \ac{SFC} or an Octomap could be extracted as required.





\section{Theory: Certified Visual Odometry (C-VO)}
\label{sec:certified-visual-odometry}

Here we describe our proposed rototranslation estimation algorithm. Recall we assume a bounded sensor measurement error of the position of each point in the pointcloud. 

Consider successive timesteps $k, k+1$. First, we use classical feature detection and optical flow algorithms to identify features in RGB image $k$, and identify their corresponding position in image $k+1$~\cite{shi1994good, lucas1981iterative}.  Once projected to 3D, we have two pointclouds with a list of correspondences. This process is not perfect, and can lead to outliers. Thus, we design a robust rototranslation algorithm that is robust to a small number of outliers. See \Cref{section:algorithms} for additional implementation details and the effect of outliers.
The \ac{VO} problem is now a pointcloud registration problem: the rototranslation between the pointclouds is equivalent to the rototranslation between successive body frames. Thus the robust rototranslation estimation problem is as follows:
\begin{problem}
\label{problem:rototranslation}
  Let $\Acal = \{a_i\}_{i=1}^{N}, \Bcal = \{ b_i \}_{i=1}^{N}$ be two sets of points $(a_i, b_i \in \R^3)$ such that $\Acal, \Bcal$ are related by the model
  \eqn{
  \label{eqn:matched_rototranslation_problem}
  b_i = R a_i + t + \epsilon_i, \quad \norm{\epsilon_i} \leq \delta_i
  }
  where $(R, t) \in \SE(3)$ is the true rototranslation between pointclouds $\Acal, \Bcal$, and $\epsilon_i \in \R^3$ is sensor noise, with known bounds $\delta_i > 0$ for each $i \in \{1, ..., N\}$. Determine $(\hat R, \hat t) \in \SE(3)$ that solves the following problem:
  \eqn{
  \argmin{\hat R \in \SO(3), \hat t \in \R^3} \sum_{i=1}^{N} \min \left( \norm{ b_i - \hat R a_i - \hat t}^2, \delta_i^2 \right) \label{eqn:tls_rototrans}
  }
\end{problem}

Equation~\eqref{eqn:tls_rototrans} is a \ac{TLS} problem, since any term with $\norm{b_i - \hat R a_i - \hat t} > \delta_i$ will only contribute a fixed amount $\delta_i^2$ to the cost. This truncated cost allows the optimization to be robust to outliers. Although nonconvex, through a sequence of reformulations, we obtain a computationally-efficient method to solve this problem. The method is largely inspired by~\cite{yang2020teaser}, but there are some key details (e.g., the computation of the error bounds) that are different. For completeness, we explain the full procedure.

\subsection{Rotation Estimation}
\label{section:rotation_estimation}
First we decouple the rotation and translation. Notice for any pair of points $i, j \in \{1, ..., N\}$, 
\eqn{
b_i - b_j  = R (a_i - a_j) + \epsilon_i - \epsilon_j
}
is independent of $t$.  Define $ a_{ij} = a_i - a_j$, $b_{ij} = b_i - b_j$, $\epsilon_{ij} = \epsilon_i - \epsilon_j$. Then, 
\eqn{
\norm{\epsilon_{ij}} \leq \norm{\epsilon_i} + \norm{\epsilon_j} \leq \delta_{ij}
}
where $\delta_{ij} =  \delta_i + \delta_j$.  To capture the pairs, we construct an undirected graph $G = (V, E)$, where $V = \{ 1, ..., N \}$ and $E \subset V \times V$. Ideally, $G$ should be a complete graph, with $\abs{E} = N(N-1)/2$ edges. In practice, other graph topologies can be used to improve computational performance, as discussed in \cref{section:graph_fraction}.
\begin{problem}
\label{problem:rotation}
    Consider the setup in~\Cref{problem:rototranslation}. Let $G = (V, E)$ be a graph connecting pairs of points. Determine $\hat R \in \SO(3)$ that solves 
    \eqn{
    \argmin{\hat R \in \SO(3)} \sum_{(i, j) \in E} \min \left( \norm{b_{ij} - \hat R a_{ij}}^2, \delta_{ij}^2 \right)
    }
\end{problem}

To solve Problem~\ref{problem:rotation}, we use \ac{GNC}~\cite{yang2020graduated}, the details of which are omitted in the interest of space. For \ac{GNC} to be efficient, we require a fast method to solve an associated \ac{WLS} problem. For Problem~\ref{problem:rotation}, the  \ac{WLS} problem has an analytic solution:
\begin{lemma}
\label{lemma:wls_rot}
    Consider \ac{WLS} problem
    \eqn{
    \argmin{\hat R \in \SO(3)} \sum_{(i,j) \in E} w_{ij} \norm{b_{ij} - \hat R a_{ij}}^2 \label{eqn:wls_rot}
    }
    given weights $w_{ij} \in [0, 1]$. Define\footnote{Recall $ \overline a_{ij} = \bmat{a_{ij}^T &  0}^T$, $\overline b_{ij} = \bmat{b_{ij}^T & 0}^T$. See \Cref{sec:notation} for $\Omega_1, \Omega_2$.}
    \begin{subequations}
    \eqn{
    Q_{ij} &= \Omega_1^T(\overline b_{ij}) \Omega_2(\overline a_{ij}) + \Omega_2^T(\overline a_{ij}) \Omega_1(\overline b_{ij})\\
    Q &= -\sum_{(i, j) \in E} w_{ij} Q_{ij} \label{eqn:wls_Q}
    }
    \end{subequations}
    where $Q_{ij}, Q \in \R^{4 \times 4}$ are symmetric matrices. Let 
    \eqn{\hat q = \texttt{eigvec}(Q) \in \R^4}
    be the (unit-norm) eigenvector corresponding to the smallest eigenvalue of $Q$. Then, $\hat q$ is the unit quaternion corresponding to the rotation matrix $\hat R$, the solution to~\eqref{eqn:wls_rot}. 
\end{lemma}
\begin{proof}
    Adapted from \cite{yang2020teaser}. See appendix~\ref{sec:proof:wls_rot}.
\end{proof}

Notice the WLS problem~\eqref{eqn:wls_rot} requires determining an eigenvector of a symmetric real matrix of fixed size $\R^{4\times 4}$ independent of $N$. This makes \Cref{problem:rotation} efficient.

Next, we establish the rotation error bound.
\begin{lemma}
\label{lemma:rotation_bound}
Consider the setup in~\Cref{problem:rotation}. Let $\hat R \in \SO(3)$ solve~\Cref{problem:rotation}. Then the rotation error is bounded by
\eqn{
\label{eqn:rotation_bound}
\fronorm{R - \hat R} \leq \epsilon_R = \sqrt{\frac{2 \norm{z}^2}{\sigma_2^2(A) + \sigma_3^2(A)} }, 
}
where 
$z \in \R^{|E|}, A \in \R^{3 \times |E|}$ are 
\eqn{
z = \bmat{ \cdots & z_{ij} & \cdots}^T, \quad 
    A = \bmat{ \cdots & \frac{a_{ij}}{\norm{a_{ij}}} & \cdots}, \label{eqn:A_matrix}
}
and $z_{ij} = \left(\norm{b_{ij} - \hat R a_{ij}} + \delta_{ij}\right)/\norm{a_{ij}}$ for each $(i,j) \in E$. $\sigma_1(A) \geq \sigma_2(A) \geq \sigma_3(A)$ are the singular values of $A$. 
\end{lemma}
\begin{proof}
Adapted from \cite[Thm.~36]{yang2020teaser}. See appendix~\ref{section:proof:rotation_bound}.
\end{proof}

\begin{remark}
    The proof of~\Cref{lemma:rotation_bound} does not depend on the algorithm used to obtain $\hat R$, but only on the bounds in~\eqref{eqn:matched_rototranslation_problem}. Thus~\Cref{lemma:rotation_bound} is applicable even with alternative algorithms. 
\end{remark}

We provide some intuition for~\eqref{eqn:rotation_bound}. First, $\epsilon_R^2$ is proportional to $\sum z_{ij}^2$. This term captures the both the re-projection error $\norm{b_{ij} - \hat R a_{ij}}$, and the assumed noise bound $\delta_{ij}$. Second, $\epsilon_R^2$ is inversely proportional to $\sigma_2^2(A) + \sigma_3^2(A)$. These singular values characterize the distribution of points - if the points all lie in a single line, $\sigma_2(A) = \sigma_3(A) = 0$ and $\epsilon_R$ tends to infinity. This corresponds to a case where there is not enough information in the pointclouds to determine the rotation. 

While the bound in \Cref{lemma:rotation_bound} is correct, we can further tighten the bounds. Consider how the error bound scales with $|E|$. To first order, the numerator $\norm{z}^2$ scales with $|E|$. The denominator contains singular values of an $3\times |E|$ matrix. The singular values will be larger when the columns of $A$ are orthogonal to each other, and therefore, only a small subset of the columns of $A$ contribute to large singular values. Thus we can achieve tighter error bounds if a small subgraph of $G$ is used. This leads to the following Lemma. 

\begin{lemma}
\label{corrollary:rotation_bound}
    Consider the setup in~\Cref{lemma:rotation_bound}. The maximum rotation error is also bounded by 
    \eqn{
    \fronorm{R - \hat R} \leq \epsilon_R = \sqrt{\frac{2 \norm{\tilde z}^2}{\sigma_2^2\left(\tilde A\right) + \sigma_3^2\left(\tilde A\right)} } \label{eqn:rotation_bound_small}
    }
    where $\tilde z \in \R^3$ and $\tilde A \in \R^{3 \times 3}$ are defined as
    \eqn{
\tilde z = \bmat{ z_{ij} &  z_{ik} & z_{il}}^T, \quad 
    \tilde A = \bmat{ \frac{a_{ij}}{\norm{a_{ij}}} & \frac{a_{ik}}{\norm{a_{ik}}} & \frac{a_{il}}{\norm{a_{il}}}}, \label{eqn:A_matrix_small}
}
    where $(i,j),(i,k), (i,l)$ are any three edges selected from graph $G$.
\end{lemma}
\begin{proof}
Consider a subgraph of $G$ consisting of only vertices $\{i, j, k, l\}$, and edges $(i,j), (i,k), (i,l)$. Using the same rotation estimate $\hat R$, apply~\Cref{lemma:rotation_bound}. This leads to~\eqref{eqn:rotation_bound_small}.
\end{proof}

Naturally, the edges should be selected to minimize $\epsilon_R$. This is an NP-hard subset selection problem~\cite{de2007subset, avron2013faster}.  Although heuristic methods exist, they require constructing $A$ in~\eqref{eqn:A_matrix} and computing its SVD. This alone takes over 10~ms. Instead, we avoid constructing $A$, randomly select four nodes of $G$, and compute $\epsilon_R$ using \Cref{corrollary:rotation_bound}. We repeat this for some number of iterations, and use the tightest bound calculated. As demonstrated in \Cref{fig:iterations}, as the number of iterations increases, the error bound gets tighter. In our implementations we perform 1000~iterations.  This takes under a millisecond, and still yields reasonably tight bounds.

\begin{figure}
    \centering
    \includegraphics[width=0.99\linewidth]{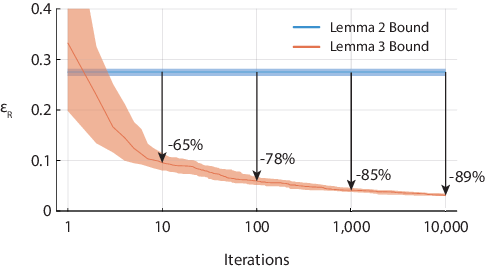}
    \caption{Rotation error bounds due to \Cref{corrollary:rotation_bound} (orange) are tighter than due to \Cref{lemma:rotation_bound} (blue). Each line shows the median and interquartile range of 100 trials of rototranslation estimation on synthetic data. }
    \label{fig:iterations}
\end{figure}



\subsection{Translation Estimation}

Once \Cref{problem:rotation} is solved,  we can solve for the translation:
\begin{problem}
\label{problem:translation}
    Consider the setup in \Cref{problem:rototranslation}. Suppose $\hat R \in \SO(3)$ solves \Cref{problem:rotation}. Determine $\hat t \in \R^3$ that solves 
    \eqn{
\argmin{\hat t \in \R^3}  \sum_{i=1}^{N} \min \left( \norm{ b_i - \hat R a_i - \hat t}^2, \delta_i^2 \right). \label{eqn:tls_trans}
}
\end{problem}

As before, the \ac{TLS} is solved using \ac{GNC}, where the corresponding \ac{WLS} problem is a standard least-squares problem. The following establishes an error bound.

\begin{lemma}
\label{lemma:translation_bound}
    Consider the setup in Problem~\ref{problem:rototranslation}. Let $\hat R \in \SO(3), \hat t \in R^3$ solve Problem~\ref{problem:rotation} and~\ref{problem:translation}. Then the translation error is bounded by
    \eqn{
    \norm{t - \hat t} \leq \epsilon_t =  \min_{i \in \{1, ..., N\}} \left( \epsilon_R \norm{a_i} + \norm{\phi_i} + \delta_i \right) \label{eqn:translation_bound}
    }
    where $\phi_i = b_i - \hat R a_i - \hat t$ and $\epsilon_R$ is defined in~\eqref{eqn:rotation_bound}.
\end{lemma}
\begin{proof}
This bound is obtained using the triangle inequality. See appendix~\ref{sec:proof:translation_bound}.
\end{proof}

To summarize, the \ac{CVO} algorithm computes the rototranslation between successive frames by decomposing it into separate rotation and translation problems. The corresponding error bounds are obtained in~\Cref{corrollary:rotation_bound} and \ref{lemma:translation_bound}.

\begin{remark}
The main difference of this work wrt to~\cite{yang2020teaser} is that we use \ac{GNC} for rototranslation estimation. \ac{GNC} is robust to a small number of outliers, and recovers the global optimal in the absence of outliers~\cite{yang2020graduated}. The QCQP verification in~\cite{yang2020teaser} could be used, but is prohibitively slow for realtime \ac{VO}. Furthermore, it cannot handle motion blur.

\Cref{lemma:wls_rot} and \ref{lemma:rotation_bound} were reported in~\cite{yang2020teaser}. \Cref{corrollary:rotation_bound} and \Cref{lemma:translation_bound} are new results that provide computationally efficient methods to obtain tighter bounds.
\end{remark}

\section{Theory: Certified Mapping}
\label{sec:certified-esdf}

This section describes our mapping algorithm. The challenge is to account for the drift in localization from a \ac{VO} algorithm, in practice often ignored by mapping algorithms. Since we do not make this simplifying assumption, we review \ac{ESDF}s using the notation of this paper. 

\subsection{Background}

Let $I$ be an inertial frame.  Let $\Ocal \subset \R^3$ be the set of obstacles (closed with no isolated points), assumed stationary in $I$. $B_k$ is the robot body frame when the $k$-th image is received. The mapping frame $M$ is (ideally) stationary relative to $I$, but due to \ac{VO} drift, $M$ shifts relative to $I$. 

The \ac{ESDF} is defined as the function $d: \R^3 \to \R$, 
\eqn{
\label{eqn:esdf}
d(p) = \min_{o \in \Ocal} \norm{ o - p},
}
the distance between the point $p$ and the nearest obstacle. To evaluate~\eqref{eqn:esdf},  $o$ and $p$ must be expressed in a common frame. Recall $p\inframe{A}$ denotes $p$ expressed in frame $A$. 
To represent the \ac{ESDF} computationally, we discretize the environment into a grid of voxels, and store the \ac{ESDF} at each voxel. Since this is done in the mapping frame, it is denoted $d_M : \R^3 \to \R$. 

For safety-critical path planning and control, we need the \ac{ESDF} at body-fixed positions. The common approach is to assume the odometry is exact, and determine $d(p\inframe{B_k})$ by expressing it in the map frame and evaluating $d_M$:
\eqn{
d(p\inframe{B_k}) \approx d_M( \hat p\inframe{M}) = d_M( \underbrace{ \estrotation{B_k}{M} p\inframe{B_k} + \esttranslation{B_k}{M}  }_{\hat p \inframe{M}} ) \label{eqn:incorrect_esdf}
}
However, since the estimate $(\estrotation{B_k}{M}, \esttranslation{B_k}{M})$ is inexact, this method can lead to over- or underestimates. Overestimated distances are unsafe since they could lead to collisions.

\subsection{Proposed Approach}

The goal is to construct an \ac{ESDF} that is safe, i.e., underestimates the distance to obstacles. More precisely, 
\begin{definition}
\label{def:correct}
    Consider an environment with obstacles $\Ocal \subset \R^3$, assumed static in frame $I$. Let the \ac{ESDF} of $\Ocal$ be $d: \R^3 \to \R$. A \emph{Certified-ESDF} (C-ESDF) at timestep $k$ is a function $d_M^k: \R^3 \to \R$, such that for all points $p\inframe{B_k} \in \R^3$, 
\eqn{
d(p\inframe{B_k}) &\geq d_M^k( \hat p\inframe{M}) = d_M^k(\estrotation{B_k}{M} p\inframe{B_k} + \esttranslation{B_k}{M}) \label{eqn:correct_esdf}
}
where $\estrototranslation{B_k}{M} \in \SE(3)$ is the estimated rototranslation between $B_k$ and $M$.
\end{definition}

Comparing \eqref{eqn:incorrect_esdf} with \eqref{eqn:correct_esdf}, the goal of certification is to change the $\approx$~into~$\geq$. That is, a Certified-ESDF is one where for any body-fixed point $p\inframe{B_k}$, if the point is expressed in the mapping frame \emph{using the estimated rototranslation}, we have \emph{underestimated} the distance to the nearest obstacle:
\eqn{
\underbrace{d(p\inframe{B_k}) = \min_{o \in \Ocal} \norm{ p\inframe{B_k} - o\inframe{B_k}}}_{\text{true SDF}} &\geq \underbrace{d_M(\estrotation{B_k}{M} p\inframe{B_k} + \esttranslation{B_k}{M})}_{\text{estimated SDF}}.
}

To accomplish this, we propose a strategy of deflating the \ac{ESDF}. We derive a recursive guarantee to ensure the \ac{ESDF} remains certified for all $k$. Note, \Cref{theorem:new_radius} only describes the deflation step: see \Cref{section:algorithms} for the complete algorithm. We start with the following assumption:
\begin{assumption}
\label{assumption:vo_is_correct}
Assume the rototranslation error between successive timesteps is bounded by known bounds
    \eqnN{
    \fronorm{\Brotation{k}{k+1} - \estBrotation{k}{k+1}} \leq \epsilon_{R, k}, \quad
    \norm{\Btranslation{k}{k+1} - \estBtranslation{k}{k+1}} \leq \epsilon_{t, k}.
    }
\end{assumption}

Notice that \Cref{assumption:vo_is_correct} is exactly the result obtained  in the  \ac{CVO} algorithm, specifically \Cref{corrollary:rotation_bound} and \Cref{lemma:translation_bound}.

Consider the following update rule to construct the $(k+1)$-th \ac{ESDF} from the $k$-th \ac{ESDF}:
\eqn{
d_M^{k+1}(p\inframe{M}) = d_M^k(p\inframe{M}) -  \underbrace{\left ( \epsilon_{R, k} \norm{ \hat p\inframe{B_{k+1}}  -  \estBtranslation{k}{k+1} } + \epsilon_{t, k}  \right) }_{\Delta} \label{eqn:new_radius}
}
for all $p\inframe{M} \in \R^3$, where $\hat p\inframe{B_{k+1}} =  \estrotation{M}{B_{k+1}} p\inframe{M} + \esttranslation{M}{B_{k+1}}$. Notice the correction to \ac{CESDF}, $\Delta$,  is different at each $p\inframe{M}$.

\Cref{theorem:new_radius} certifies the new \ac{ESDF}:

\begin{theorem}
\label{theorem:new_radius}
Let \Cref{assumption:vo_is_correct} hold. Suppose at timestep $k\in \naturals$ the \ac{ESDF} $d_M^k : \R^3 \to \R$ is a Certified-ESDF by \Cref{def:correct}, 
i.e., $\forall p\inframe{B_k} \in \R^3$ 
\eqn{
d(p\inframe{B_{k}}) \geq d_M^{k}(\underbrace{\estrotation{B_k}{M} p\inframe{B_k} + \esttranslation{B_k}{M}}_{ p\inframe{B_k} \text{ expressed in } M } ).
}
Then, if the ($k+1$)-th \ac{ESDF} is constructed using \eqref{eqn:new_radius}, the $(k+1)$-th ESDF is  a Certified-ESDF. That is, $\forall p\inframe{B_{k+1}} \in \R^3$
\eqn{
    d(p\inframe{B_{k+1}}) \geq d_M^{k+1}(\underbrace{\estrotation{B_{k+1}}{M} p\inframe{B_{k+1}} + \esttranslation{B_{k+1}}{M}}_{ p\inframe{B_{k+1}} \text{ expressed in } M} ). \label{eqn:conclusion}
    }
    
\end{theorem}

\begin{proof}[Proof of \Cref{theorem:new_radius}]
A condensed proof is provided here. See Appendix~\ref{section:proof:new_radius} for more details. Consider any point $p\inframe{B_{k+1}} \in \R^3$. When expressed in frame $B_k$, we have
\eqnN{
&\norm{p\inframe{B_k} - \hat p\inframe{B_k}}\\
&= \norm{ \left( \Brotation{k+1}{k}p\inframe{B_{k+1}} + \Btranslation{k+1}{k} \right) - \left(
        \estBrotation{k+1}{k} p\inframe{B_{k+1}} + \estBtranslation{k+1}{k} \right)}\\
&= \Big \Vert \left( \Brotation{k+1}{k}- \estBrotation{k+1}{k} \right) \left( p\inframe{B_{k+1}} -  \estBtranslation{k}{k+1}   \right) \\
& \quad \quad + \Brotation{k+1}{k} \left( - \Btranslation{k}{k+1} +  \estBtranslation{k}{k+1} \right) \Big \Vert
}
by adding and subtracting $\Brotation{k+1}{k} \estBtranslation{k}{k+1}$ inside the norm.
Using the triangle inequality, 
\eqnN{
&\norm{p\inframe{B_k} - \hat p\inframe{B_k}}\\
&\leq \norm{ \Brotation{k+1}{k} - \estBrotation{k+1}{k}} \norm{p\inframe{B_{k+1}} -  \estBtranslation{k}{k+1} } + \norm{ \Btranslation{k}{k+1} -\estBtranslation{k}{k+1}  }\\
&\leq \epsilon_{R, k} \norm{p\inframe{B_{k+1}} -  \estBtranslation{k}{k+1} } + \epsilon_{t, k}
}

Thus, when $p\inframe{B_{k+1}}$ is expressed in $B_k$, it could correspond to any point in $\Pcal = \{ \tilde p \in \R^3 : \norm{ \tilde p - \hat p\inframe{B_{k}}} \leq \Delta \}$ where $\Delta = \epsilon_{R, k} \norm{p\inframe{B_{k+1}} -  \estBtranslation{k}{k+1} } + \epsilon_{t, k} $, and  $\hat p\inframe{B_k} = \estBrotation{k+1}{k} p\inframe{B_{k+1}} + \estBtranslation{k+1}{k}$.
Then, using properties of an SDF,
\eqnN{
d(p\inframe{B_{k+1}}) &\geq \min_{\tilde p \in \Pcal} d(\tilde p) \geq d(\hat p\inframe{B_k}) - \Delta \geq d_M^k(\hat p\inframe{M}) - \Delta
}
where $\hat p\inframe{M} = \estrotation{B_k}{M} \hat p\inframe{B_k} + \esttranslation{B_k}{M}$.  Therefore, we have
\eqnN{
d(p\inframe{B_{k+1}}) &\geq d_M^k(\hat p\inframe{M}) - \epsilon_{R, k} \norm{p\inframe{B_{k+1}} -  \estBtranslation{k}{k+1} } - \epsilon_{t,k}\\
&= d_M^{k+1}(\hat p\inframe{M}) 
}
completing the proof. 
\end{proof}

\Cref{fig:cesdf-diagram} demonstrates \Cref{theorem:new_radius}.  Consider a static obstacle 1~m away from the robot at the initial time. Suppose the robot moves by 0.1~m between frames $k=0$ and $1$. Suppose we need the \ac{ESDF} at a body-fixed point $p\inframe{B_1} = [0.05, 0.4]^T$.  The true ESDF is thus $0.1 + 0.4 = 0.5$~m.  However, suppose the estimated rototranslation was not exact, with a 5$^\circ$ rotation error and a 2~cm position error, as depicted on the right. 

In the approximate \ac{ESDF} approach~\eqref{eqn:incorrect_esdf}, we would express $p\inframe{B_1}$ in frame $M$, and compute the \ac{ESDF} to be 0.526~m, greater than the true ESDF. The approximate approach yields an unsafe estimate of the \ac{ESDF}. 

On the other hand, consider the proposed approach, i.e.,~\eqref{eqn:new_radius}. Given the error, we have $\epsilon_R = 0.123$ (\Cref{lemma:fronorm_angular_norm}) and $\epsilon_t = 0.02$. Using~\eqref{eqn:new_radius} the correction is $\Delta = 0.070$~m. Computing the corrected \ac{ESDF} yields $d_M^1(\hat p \inframe{M}) = 0.456$~m. Since $0.456 \leq 0.5$, this C-ESDF is an underestimate of the true ESDF.  This demonstrates safe behavior that satisfies our definition of a Certified ESDF.

\begin{figure}
    \centering
    \includegraphics{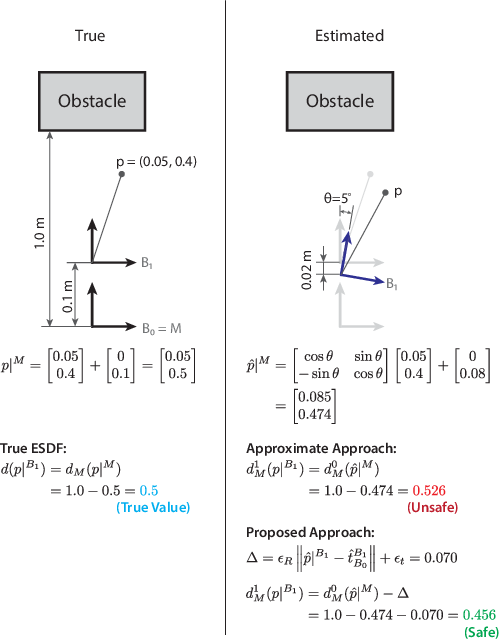}
\caption{Diagram demonstrating the \ac{CESDF} approach. Suppose the robot moves by 0.1~m in the y-axis, but the estimated rototranslation has some error. The approximate approach~\eqref{eqn:incorrect_esdf} is unsafe, estimating the ESDF of $p$ to be 0.526~m, greater than the true ESDF of 0.5~m. In the proposed approach~\eqref{eqn:new_radius}, we calculate a correction $\Delta$, and compute the C-ESDF to be 0.456~m, a safe underestimate. }
    \label{fig:cesdf-diagram}
\end{figure}

\begin{remark}
    The form of \eqref{eqn:new_radius} can be interpreted as follows. The update to the \ac{ESDF} is always a correction:
    \eqnN{
\underbrace{d_M^{k+1}(p)}_{\text{new ESDF}} = \underbrace{d_M^k(p)}_{\text{previous ESDF}}  \underbrace{- \epsilon_{R, k} \norm{ \hat p\inframe{B_{k+1}} - \estBtranslation{k}{k+1} }}_{\text{rotation correction}} \underbrace{- \epsilon_{t, k}}_{\text{translation correction}}
}

With a non-zero translation error the \ac{ESDF} must be decremented everywhere by $\epsilon_t$. With a non-zero rotation error, every point must also be decremented, but the decrement increases with the distance from the camera origin: a 2$^\circ$ rotation error leads to a $2\pi/180 \approx 0.03$~m error for a point 1~m away, but a 0.3~m error for a point 10~m away.
\end{remark}

\section{Algorithms and Implementation Details}
\label{section:algorithms}

Here we use the theoretical results of~\Cref{sec:certified-visual-odometry} and~\ref{sec:certified-esdf} to construct algorithms for \Cref{problem:overall}. Our implementations are open-sourced at \ifx\anonymized\undefined
\url{https://github.com/dasc-lab/certified-perception}.
\else
\emph{[redacted]}.
\fi

\subsection{C-VO}

The first step is to identify and track features in successive RGBD images. We use Good Features to Track \cite{shi1994good}, and the Lukas-Kanade optical flow algorithm \cite{lucas1981iterative}, since these represent well established baseline methods. By performing the feature detection and mapping in the 2D image space, we avoid the complexities of feature selection and matching in 3D pointclouds~\cite{tombari2013performance, choy2019fully}. In our experiments, we observe approximately 300 features matched per frame and a 1-5\% outlier rate, depending on motion blur, lighting conditions, and the richness of the observed scene. 

\label{section:graph_fraction}
At 300 features per frame, a complete graph $G$ would have $|E| = 44,850$ edges. This is computational bottleneck in~\eqref{eqn:wls_Q}. Empirically, we observe that only a small fraction $f \in (0, 1)$ of edges are needed to recover the true solution of the rotation estimation problem~\Cref{fig:vo-time-accuracy}. The figure shows that as $f$ reduces, the computation time decreases, but the accuracy is largely unaffected. We use $f=5$\%, as a compromise between runtime, accuracy, and robustness.

\begin{figure}
    \centering
    \includegraphics[width=0.95\linewidth]{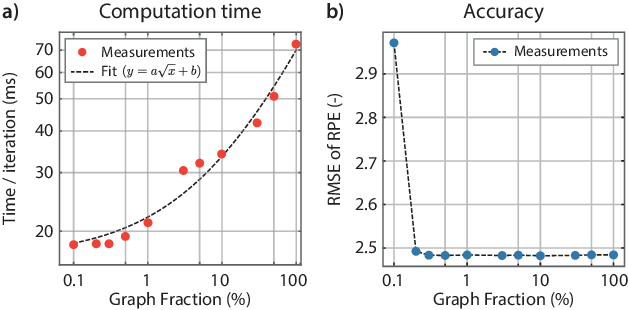}
    \caption{Effect of graph fraction on accuracy and runtime of Certified Visual Odometry (CVO). (a)~Log-Log plot of the computation time. (b) Log-Linear plot of accuracy. Both graphs were produced by running the algorithm on a recorderd dataset using the Realsense D455 camera, and ground-truth provided by VICON.}
    \label{fig:vo-time-accuracy}
\end{figure}

\begin{algorithm}
  \label{alg:vo}
  \DontPrintSemicolon
  \caption{Certified Visual Odometry}
  \SetKw{KwParams}{Parameters:}
  \SetKwProg{When}{When}{}{end}
  \SetKwProg{Fn}{function}{}{end}
  \When{received $(k+1)$-th RGBD image}
  {
  Using optical flow, identify corresponding features in RGB frames $k$ and $k+1$.\;
  Project features to 3D using the depth images. \;
  Construct graph $G$ connecting pairs of features.\;
  Determine $\hat R_{k}^{k+1}, \hat t_{k}^{k+1}$ using \eqref{eqn:wls_rot}, \eqref{eqn:tls_trans} and \ac{GNC}. \;
  Determine $\epsilon_{R, k}, \epsilon_{t, k}$ using~\eqref{eqn:rotation_bound_small}, \eqref{eqn:translation_bound}. \;
  Publish  $(\hat R_{k}^{k+1}, \hat t_{k}^{k+1}, \epsilon_{R, k}, \epsilon_{t, k})$.\;
  }
\end{algorithm}

\subsection{C-ESDF}

Our implementation of the \ac{CESDF} builds on~\cite{nvblox}. We describe our modifications, assuming familiarity with~\cite{oleynikova2017voxblox, nvblox}. 

As in~\cite{nvblox}, we maintain two \ac{SDF}s, the \ac{TSDF} and the \ac{ESDF}. The \ac{TSDF} only contains values near the obstacle geometry, and is used to determine the zero level-set of the obstacles. When a new depth images are received, we use raycasting to update the \ac{TSDF} of voxels within the \ac{FOV}. From the updated cells, we propagate waves to construct the \ac{ESDF}. 

Our modification is to add a deflation step, based on \Cref{theorem:new_radius}. We keep track of a \texttt{correction} in each voxel, the cumulative sum of decrements. When a voxel lies within the \ac{FOV}, the correction is reset to zero, allowing new regions to be added to the free set. When publishing the \ac{CESDF}, we subtract the correction from estimated \ac{ESDF}. Since the decrement at each voxel is independent of every other voxel, the decrement step can be parallelized efficiently on the GPU. 

\begin{algorithm}
  \label{alg:cert-esdf}
  \DontPrintSemicolon
  \caption{Certified Mapping}
  \SetKw{KwParams}{Parameters:}
  \SetKwProg{When}{When}{}{end}
  \SetKwProg{Every}{Every}{}{end}
  \SetKwProg{Fn}{function}{}{end}
  \When{received $(k+1)$-th RGB-D image}
  {
  Determine $\estBrototranslation{k}{k+1}$ and $\epsilon_{R,k}, \epsilon_{t, k}$ using Algorithm~\ref{alg:vo}. \;
  Update the \ac{TSDF} within camera's FOV.\;
  Update \texttt{correction} of all voxels using Thm.~\ref{theorem:new_radius}.\;
  For all voxels in FOV, reset \texttt{correction} to 0.\;
  }
  \Every{$T$ seconds}
  {
  Propagate waves from all surface voxels to update the \ac{CESDF}.\; 
  }
\end{algorithm}

\section{Experimental Evaluation}
\label{sec:experiments}

Here we report our experiments. As part of developing these libraries, extensive tests were performed on synthetic data and test datasets, but omitted in the interest of space. Only results from hardware experiments are reported. 

\subsection{C-VO}

\begin{figure*}
    \centering
    \includegraphics[width=0.9\linewidth]{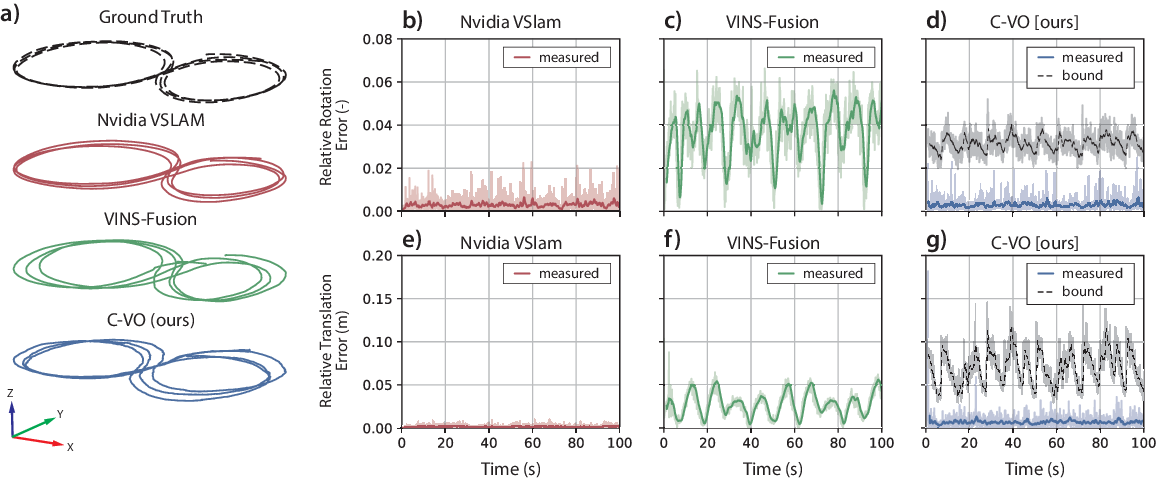}
    \caption{Comparison of the Certified Visual Odometry (C-VO) algorithm with benchmark algorithms. We compare \ac{CVO} with Nvidia VSlam~\cite{vslam} and VINS-Fusion~\cite{yu2021vins}. To compute the error, we also recorded the ground-truth trajectory using a VICON motion capture system. To enable a fair comparison, each method was run in Visual Odometry only mode, i.e., without loop-closures or IMU measurements. (a) The reconstructed trajectory from each method. (b-g) The relative rotation error and relative translation error for each method. Only our method provides error bounds.}
    \label{fig:cvo-error}
\end{figure*}

\begin{table*}
    \centering
    \caption{Comparison of \ac{CVO} with state-of-the-art algorithms. }
    \begin{tabular}{lccc}
    \toprule
        Method & RMS Rotation Error (-) & RMS Translation Error (cm) &  Time / iteration (ms) \\
        \midrule
      V-SLAM~\cite{vslam} & 0.00376 & 0.242 &  1.40\\
        VINS-Fusion \cite{yu2021vins} & 0.03941 & 3.102 & 39.24\\
         C-VO [ours] & 0.00380 & 0.882 & 14.24\\
         \midrule
         C-VO [ours] Error Bound & 0.03177 & 7.239 & (included above)\\
         \bottomrule
    \end{tabular}
    \label{table:cvo-error}
\end{table*}

\begin{table}
    
\end{table}

\Cref{fig:cvo-error} and \Cref{table:cvo-error} compare the performance of the \ac{CVO} method to two state-of-the-art methods for Visual Odometry, Nvidia VSLAM~\cite{vslam} and VINS-Fusion~\cite{yu2021vins}. Note, for a fair comparison, we turn off the IMU and loop-closure components of each algorithm. We use VICON for ground-truth. The following metrics measure the relative rototranslation error,\footnote{These definitions differ to those in~\cite{grupp2017evo}, since these correspond to $\epsilon_R, \epsilon_t$.}
\eqnN{
RRE_k &= \fronorm{ \estBrotation{k}{k+1} - \Brotation{k}{k+1}},\\
RTE_k &= \norm{ \hat t_{B_k}^{B_{k+1}} - t_{B_k}^{B_{k+1}}},
}
where the quantities with and without $\hat{(\cdot)}$ denote the estimates and the ground truths respectively. 

The results show that our \ac{CVO} algorithm produces an rotation error that is on-par with VSLAM, but accumulates 6mm of additional translation error per frame (in terms of the RMSE). Our algorithm is more accurate than VINS-Fusion, with an order of magnitude lower rotation error, and a third of the translation error (\Cref{table:cvo-error}). \Cref{table:cvo-error} also shows that VSLAM implementation is 10x faster than ours, and 30x faster than VINS-Fusion, likely due to specialized GPU code. This makes the entire (closed-source) VSLAM odometry algorithm faster than the OpenCV feature detection step alone. Improving our implementation should allow us to track more features and therefore improve performance.

The primary benefit of our algorithm is that it is able to produce error bounds, plotted in Figs.~\ref{fig:cvo-error}d, \ref{fig:cvo-error}g. The error bound is always greater than the measured error. The error bound is (in terms of the RMSE) about an order of magnitude greater than the measured error. This is to be expected, since our error bounds are calculated assuming worst case disturbances, and extensive use of the triangle inequality.

\subsection{C-ESDF}

\begin{figure*}[t]
    \centering
    \includegraphics[width=0.9\linewidth]{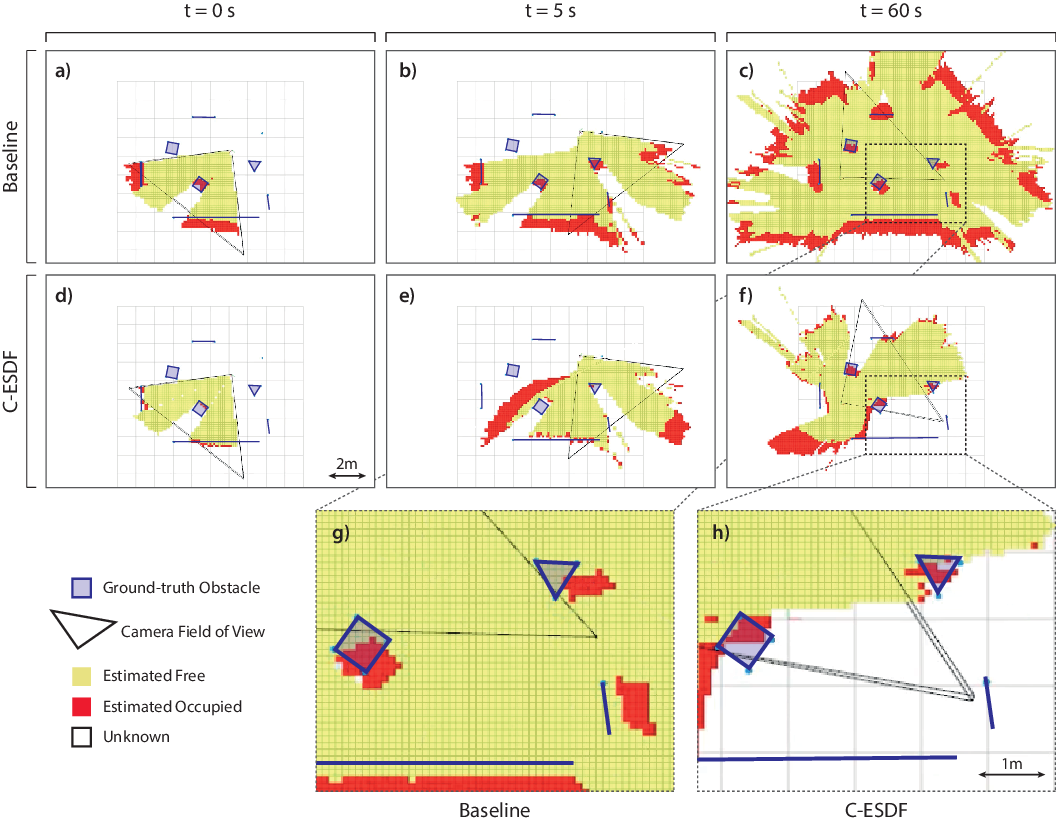}
    \caption{Certified Mapping. (a-c) Snapshots of the \ac{ESDF} using the baseline algorithm. (d-f) Snapshots using the proposed \ac{CESDF} algorithm. (g, h) Insets showing (c, f) magnified. Each figure shows a slice at $z=1.0$~m of the constructed 3D \ac{ESDF}. The figures show that while the baseline \ac{ESDF} produces a map of entire lab space, while the \ac{CESDF} produces a local map. However, due to the accumulated error, the baseline \ac{ESDF} misclassified some regions to be free of obstacles, while the \ac{CESDF} does not have this issue.}
    \label{fig:cert-mapping}
\end{figure*}

\Cref{fig:cert-mapping} shows the performance of the Certified \ac{ESDF} algorithm (bottom row), compared to the  baseline implementation from Nvidia~\cite{nvblox} (top row). The results are easier to understand in the supplementary video, at \ifx\anonymized\undefined
\url{https://github.com/dasc-lab/certified-perception}.
\else
[redacted].
\fi

In all figures we show the scene in the mapping frame. Due to \ac{VO} drift the obstacles are not stationary in the mapping frame. Using VICON we plot the ground-truth location of the obstacles relative to the camera. The objective is to build a map where for all voxels inside the blue shapes, the map is either unknown (white) or occupied (red), but not free (yellow).

Initially the maps generated are the same (Figs.~\ref{fig:cert-mapping}a, \ref{fig:cert-mapping}d). Between $0$ and $5$~s, the camera yaws by 90$^\circ$, and starts seeing a new region (Figs. \ref{fig:cert-mapping}b, \ref{fig:cert-mapping}e).  In the baseline (Fig. \ref{fig:cert-mapping}b) the region outside the camera's \ac{FOV} is not updated. In the \ac{CESDF}, since the left region is unobserved, the deflation step decrements the ESDF, and some cells are marked red (Fig. \ref{fig:cert-mapping}e). 

After $60$~s, a large portion of the lab has now been mapped (Figs. \ref{fig:cert-mapping}c, \ref{fig:cert-mapping}f). The accumulated \ac{VO} drift leads to errors in the baseline map. In particular, consider the magnified insets (Figs. \ref{fig:cert-mapping}g, \ref{fig:cert-mapping}h). In Fig. \ref{fig:cert-mapping}g we see some yellow cells that should be red. In contrast, in Fig. \ref{fig:cert-mapping}h, the obstacles are marked correctly.\footnote{Note, there are some cells in the \ac{CESDF} that are still incorrect, due to numerical issues and the resolution used.} Notice the forgotten regions are not based on either the distance of the point to the origin of the camera, or the time that elapsed since the point was in the camera's \ac{FOV}. These are two common heuristics used, but the \ac{CVO} algorithm deflates the \ac{ESDF} based on the actual odometry drift. 

Overall, we see that with the \ac{CESDF} algorithm, the region of the map that is classified as either free or obstacle is smaller (i.e., greater parts of the map are unknown). However, the regions that are marked as free are indeed free relative to the ground truth obstacle geometry. Furthermore, with more accurate \ac{CVO}, the deflation rate will be smaller.

\subsection{Timing Analysis}
\label{section:timing}

\begin{table}[]
    \centering
    \caption{Timing Breakdown. All steps except for the ESDF wave propagation happen at each new camera frame, i.e., 30~Hz. The wave propagation occurs at 5~Hz.}
    \begin{tabular}{llc}
    \toprule
    Module & Step & Mean Time (ms) \\
    \midrule
    C-VO & Feature Detection & 2.804 \\
	& Optical Flow & 0.935\\
	& Estimate Rototranslation & 1.463\\
	& Compute Error Bounds & 2.365\\
C-ESDF &	Integrate Depth into TSDF & 0.674\\
	& Deflate ESDF & 0.282\\
 & ESDF Wave Propagation* & 5.491\\
 \bottomrule
    \end{tabular}
    \label{tab:timing}
\end{table}

\Cref{tab:timing} breaks down the computation time for major steps of \ac{CVO} and \ac{CESDF}. The cameras are operating at 30FPS, i.e., all processing must occur within 33~ms. We see that the \ac{CVO} algorithm takes about 7~ms, of which half is the feature detection and optical flow. The \ac{CESDF} algorithm is fast, taking about 1~ms to integrate a new depth image and to deflate the ESDF. By distributing the deflation step over the GPU, the deflation step can be performed very efficiently. The proposed methods can run in realtime.

All reported results are from experiments run on a Ryzen7 5800h 16GB, with a 3050Ti.  We have performed similar experiments on the Xavier NX 16 GB. On the Xavier, the \ac{CVO} module takes approximately 18~ms, and therefore can still run in realtime. The computation time depends heavily on the number of features observed in each frame, and therefore the variance of the compute time can be significant. 

\subsection{Limitations and Future Work}

There are a few directions in which this work can be extended. (A)~Improvements in feature detection and tracking in the RGB images are required for the \ac{CVO} algorithm to be reliable when executing aggressive maneuvers. (B)~Incorporating information from an IMU or keeping track of multiple sets of pointclouds could significantly improve the robustness of the \ac{CVO} algorithm. (C)~A semantic or geometric segmentation of the received images could be used to mask dynamic obstacles from the environment. It could also be used in the \ac{CESDF}, as the deflation correction around dynamic obstalces could be increased based on the maximum speed of the obstacles. (D)~The \ac{CVO} error bounds depend explicitly on the singular values of $A$ in~\eqref{eqn:rotation_bound}. Using the \ac{ESDF}, one could predict the regions of the map that lack features and avoid them in the path planner. This can help keep the \ac{CVO} error small. 

\section{Conclusion}

This paper has taken an  initial step towards building certified perception algorithms applicable to safety critical planning and control. Our goal was to use the visual information from an RGBD camera, construct a pose estimate, and build a map of the obstacle geometry. The generated map needs to be correct in a body-fixed frame, i.e., to ensure that regions marked as free are indeed free relative to the robot. 

To achieve this, we first developed a Certified Visual Odometry (C-VO) algorithm. This casts visual odometry as a pointcloud registration problem. We propose a robust truncated least squares algorithm to solve for the rotation and the translation between successive camera frames, and use the geometric properties to derive an error bound on rototranslation estimate. 

The second step was to devise a Certified Mapping algorithm. We used  the \ac{CVO} error bound to deflate the signed distance field representing the world. By choosing to deflate safe regions (instead of inflating known obstacles) we allow the environment to be continuously explored, and allow new information from sensors to be assimilated into the map. The decrement in the signed distance field at each point of the map is, again, derived using on geometric properties.

Importantly, the \ac{CVO} and \ac{CESDF} algorithms integrate naturally. The error bound on \ac{CVO} takes the form of a norm ball, and the same norm ball is used to deflate the \ac{ESDF}. 

Finally, we have experimental demonstrations of both the \ac{CVO} and \ac{CESDF} algorithms indicate their ability to run in realtime, although they can still be made more efficient.

\appendices

{
\setstretch{0.95}
\bibliographystyle{IEEEtran}
\bibliography{biblio, IEEEabrv}
}

\onecolumn

\section{Notation and Additional Symbols}
\begin{table}[h]
    \centering
    \caption{}
    \begin{tabular}{c p{4.5in}}
    \toprule
         Symbol & Description \\
         \midrule
         $I$ & Inertial Frame\\
         $M$ & Mapping Frame\\
         $B_k$ & Body Frame at timestep $k$\\
         $p\inframe{F}$ & Point $p$ expressed in frame $F$\\
         $\Ocal \subset \R^3$ & Obstacle set\\
         $d(p)$ & SDF (frame-independent)\\
         $d_M(p\inframe{M})$ & SDF expressed in frame $M$\\
         $d_M^k(p\inframe{M})$ & SDF expressed in frame $M$ at timestep $k$\\
         $\epsilon_{R, k}$ & Rotation error bound between frames $B_{k}$ and $B_{k+1}$\\
         $\epsilon_{t, k}$ & Translation error bound between frames $B_k$ and $B_{k+1}$\\
         \bottomrule
    \end{tabular}
    
    \label{tab:notation}
\end{table}

\section{Proofs}

\subsection{Proof of~\Cref{lemma:wls_rot}}

\label{sec:proof:wls_rot}

\begin{proof}
For ease  of notation let $k = ij$ be used to index edge $(i,j)$ of the graph $G$. Reformulating \eqref{eqn:wls_rot} in terms of quaternions, we have
\eqnN{
\argmin{q \in \S^3} \sum_{k=1}^{|E|} w_k \norm{ \overline b_k - q \circ \overline a_k \circ  q^{-1}}^2 \label{eqn:wls_quat}
}
where $\overline a_k = \bmat{a_k^T & 0}^T$, $\overline b_k = \bmat{b_k^T & 0}^T$. Since
\eqnN{
&w_k \norm{ \overline b_k - q \circ \overline a_k \circ  q^{-1}}^2 \\
&\quad = w_k \norm { \overline b_k - \Omega_2^T(q) \Omega_1(q) \overline a_k}^2 \\
&\quad = w_k ( (\overline b_k - \Omega_2^T(q) \Omega_1(q) \overline a_k)^T (\overline b_k - \Omega_2^T(q) \Omega_1(q) \overline a_k) )\\
&\quad = w_k ( \overline b_k^T \overline b_k - \overline b_k^T \Omega_2^T(q) \Omega_1(q) \overline a_k - \overline a_k^T \Omega_1^T(q) \Omega_2(q) \overline b_k + \overline a_k^T \overline a_k)\\
&\quad = w_k ( \norm{\overline b_k}^2 - \overline b_k^T \Omega_2^T(q) \Omega_1(q) \overline a_k - \overline a_k^T \Omega_1^T(q) \Omega_2(q) \overline b_k + \norm{\overline a_k}^2)\\
&\quad = w_k ( \norm{\overline b_k}^2 - q^T \Omega_1^T(\overline b_k) \Omega_2(\overline a_k) q - q^T \Omega_2^T(\overline a_k) \Omega_1(\overline b_k) q + \norm{\overline a_k}^2) \\
&\quad = w_k ( \norm{a_k}^2 + \norm{b_k}^2) - q^T ( w_k ( \Omega_1^T(\overline b_k) \Omega_2(\overline a_k) + \Omega_2^T(\overline a_k) \Omega_1(\overline b_k) ) ) q
}
Since $q$ does not appear in the first term, it can be dropped from the optimization problem, and~\eqref{eqn:wls_quat} is equivalent to
\eqnN{
\argmin{q \in \S^3} \ q^T Q q
}
where $Q \in \R^{4\times 4}$ is given in~\eqref{eqn:wls_Q}.

Finally, by interpreting the quaternion $q\in \S^3$ as a unit vector $q \in \R^4$, we have
\eqnN{
\argmin{q \in \S^3} \ q^T Q q = \argmin{q \in \R^4}  \ q^T Q q \ \text{s.t.} \norm{q} = 1
}
which is the canonical eigenvector problem.  
\end{proof}

\subsection{Proof of~\Cref{lemma:rotation_bound}}
\label{section:proof:rotation_bound}
\begin{proof}
    The proof is based on \cite[Thm. 37]{yang2020teaser}, but is repeated here in the notation of this paper. Note, we do not consider scale errors or adversarial outliers. Given any $(i, j) \in E$, 
    \eqnN{
    b_{ij} = R a_{ij} + \epsilon_{ij}
    }
    where $\norm{\epsilon_{ij}} \leq \norm{\epsilon_i} + \norm{\epsilon_j} \leq \delta_{ij} = \delta_i + \delta_j$. 

    If we define 
    \eqnN{
    \phi_{ij} = b_{ij} - \hat R a_{ij}
    }
    then, $b_{ij} = \hat R a_{ij} + \phi_{ij}$. Equating the two expressions for $b_{ij}$, 
    \eqnN{
    (R - \hat R) a_{ij} &= \phi_{ij} - \epsilon_{ij}\\
    \therefore \norm{(R - \hat R) \frac{a_{ij}}{\norm{a_{ij}}}} &= \frac{\norm{\phi_{ij} - \epsilon_{ij}}}{\norm{a_{ij}}} \leq \frac{\norm{\phi_{ij}} + \delta_{ij}}{\norm{a_{ij}}}
    }
    Define $A \in \R^{3 \times |E|}$ as the matrix obtained by horizontally stacking each normalized $a_{ij}$:
    \eqnN{
    A = \bmat{ \cdots & \frac{a_{ij}}{\norm{a_{ij}}} & \cdots }
    }
    Then, using the properties of the Frobenius norm, 
    \eqnN{
    \fronorm{(R - \hat R) A}^2 \leq \sum_{(i,j) \in E} z_{ij}^2
    }
    where $z_{ij} = (\norm{\phi_{ij}} + \delta_{ij})/\norm{a_{ij}}$.
    
    The objective is to lower-bound $\fronorm{R - \hat R}$.  Consider the matrix 
    \eqnN{
    Z &= (R - \hat R)^T (R - \hat R)\\
      &= 2I - (S^T + S)
    }
    where $S = \hat R^T R \in \SO(3)$. Since $S \in \SO(3)$, let the axis-angle representation of $S$ be $(s, \theta)$, where $s \in \R^3$ is a unit vector, and $\theta \in \R$ is the angular distance between $R, \hat R$. Thus, the axis-angle representation of $S^T$ is $(s, -\theta)$. Therefore, using the Rodrigues formula, 
    \eqnN{
    S &= \cos(\theta) I + \sin(\theta)[s]_\times + (1- \cos \theta) s s^T\\
    S^T &=  \cos(\theta) I - \sin(\theta)[s]_\times + (1- \cos \theta) s s^T
    }
    where $[s]_\times \in \R^{3\times 3}$ is the skew-symmetric matrix associated with $s \in \R^3$. Inserting these equations into $Z$, we have
    \eqnN{
    Z &= 2(1- \cos(\theta)) I - 2 (1 - \cos\theta) s s^T
    }
    Next, we investigate the eigenvalue decomposition of $Z$. Notice $s$ is an eigenvector of $Z$, associated with an eigenvalue of $0$:
    \eqnN{
    Zs &= 2(1- \cos(\theta)) s - 2 (1 - \cos\theta) s s^T s\\
    &= 2 (1- \cos(\theta)) s - 2 (1- \cos\theta) s =  0 
    }
    since $s$ is a unit vector, i.e., $s^T s = 1$. Now consider any two vectors $w_1, w_2 \in \R^3$ perpendicular to $s$ and each other. For either $w \in \{w_1, w_2\}$, 
    \eqnN{
    Zw &= 2(1- \cos(\theta)) w - 2 (1 - \cos\theta) s s^T w\\
    &= 2(1- \cos(\theta)) w 
    }
    since $s^T w = 0$. Thus, the eigenvalues of $Z$ are $\{\lambda, \lambda, 0\}$ where $\lambda = 2 (1-\cos \theta)$. Since $Z^{1/2} = (R - \hat R)$,  the singular values of $(R - \hat R)$ are $\{ \sqrt{\lambda}, \sqrt{\lambda}, 0\}$. 

    Now let the \ac{SVD} of $(R - \hat R)$ be $U \Sigma V^T$. Then, 
    \eqnN{
    &\fronorm{( R - \hat R)^T A}^2 = \fronorm{U \Sigma V^T A }^2 = \fronorm{\Sigma V^T A }^2\\
    =& \fronorm{\bmat{\sqrt{\lambda} \\ & \sqrt{\lambda} & \\ & & 0} \bmat{ (V^T A)_1 \\ (V^T A)_2 \\ (V^T A)_3 } }^2\\
    =& \lambda \left( \norm{(V^TA)_1}^2 + \norm{(V^TA)_2}^2 \right)\\
    =& \lambda \left( \fronorm{V^TA}^2 - \norm{(V^TA)_3}^2 \right)
    }
    where the notation $(V^TA)_1$ denotes the first row of $V^TA$. Recall singular values of a matrix $X \in \R^{3 \times 3}$ are listed in non-increasing order: $\sigma_1(X) \geq \sigma_2(X) \geq \sigma_3(X) \geq 0$. 
    Since $V$ is a unitary matrix, and  $\fronorm{XY}^2 \leq \sigma_{1}^2(Y) \fronorm{X}^2$ for any two matrices $X, Y$, we have
    \eqnN{
    \norm{(V^TA)_3}^2 &= \fronorm{\bmat{ 0  \\ &0 \\ &&1} V^T A}^2 \\
    &\leq \sigma_{1}^2(A) \fronorm{\bmat{0  \\ &0 \\ &&1} V^T}^2\\
    \therefore \norm{(V^TA)_3}^2 & \leq \sigma_{1}^2(A)
    }
    Now using the property $\fronorm{V^TA}^2 = \fronorm{A}^2 = \sigma_1^2(A) + \sigma_2^2(A) + \sigma_3^2(A)$, we have
    \eqnN{
    \fronorm{( R - \hat R)^T A}^2 &= \lambda\left ( \norm{V^TA}_F^2 - \norm{(V^TA)_3}^2 \right)\\
    & \geq \lambda\left( \sigma_1^2(A) + \sigma_2^2(A) + \sigma_3^2(A) - \sigma_{1}^2(A) \right)\\
    & \geq \lambda\left( \sigma_2^2(A) + \sigma_3^2(A) \right)
    }
    Combining the upper and lower bounds on $\norm{(R- \hat R)^T A}_F$, 
    \eqnN{
    \lambda( \sigma_2^2(A) + \sigma_3^2(A) ) &\leq \norm{( R - \hat R)^T A}_F^2 \leq \sum_{(i,j) \in E} z_{ij}^2\\
\therefore    \lambda &\leq \frac{1}{\sigma_2^2(A) + \sigma_3^2(A)} \sum_{(i,j) \in E} z_{ij}^2
    }
    This leads to the Frobenius bound:
    \eqnN{
    \norm{R - \hat R}_F^2 &= \operatorname{tr}((R - \hat R)(R - \hat R)^T)\\
    &= \operatorname{tr}(2 I - R \hat R^T - \hat R R^T)\\
    &= 6 - 2 \operatorname{tr}(R \hat R^T)\\
    &= 6 - 2 (1 + 2 \cos \theta)\\
    &= 2 \lambda\\
    &\leq \frac{2}{\sigma_2^2(A) + \sigma_3^2(A)} \sum_{(i,j) \in E} z_{ij}^2
    }
    since $\lambda = 2 (1 - \cos \theta)$. 
\end{proof}

\begin{lemma}
\label{lemma:fronorm_angular_norm}
    Given $R_1, R_2 \in \SO(3)$, the error between the rotations can be  expressed equivalently as a Frobenius norm or as an angular error:
    \eqn{
    \cos \theta = 1 - \frac{\norm{R_1 - R_2}_F^2}{4}.
    }
\end{lemma}
\begin{proof}
\eqn{
    \norm{R_1 - R_2}_F^2 &= \tr{(R_1 - R_2)(R_1 - R_2)^T}\\
    &= \tr{2 I - R_1 R_2^T - R_2 R_1^T}\\
    &= 6 - 2 \tr{R_1 R_2^T}
    }
    Now defining the relative rotation from $R_1$ to $R_2$ as $\Delta = R_1 R_2^T \in \SO(3)$, the minimum angle corresponding to the rotation $\Delta$ is $\theta \geq 0$ about some axis $u \in \R^3$. Then, $\tr{\Delta} = 1 + 2 \cos \theta$ according to~\cite[Appendix A]{huynh2009metrics}. Therefore, 
    \eqn{
    \norm{R_1 - R_2}_F^2 &= 6 - 2 (1 + 2 \cos \theta)\\
    &= 4 (1- \cos \theta)
    }
\end{proof}

\subsection{Proof of~\Cref{lemma:translation_bound}}
\label{sec:proof:translation_bound}
\begin{proof}
    Consider any $i \in \Ical$. Then, 
    \eqnN{
    b_i = R a_i + t + \epsilon_i\\
    b_i = \hat R a_i + \hat t + \phi_i.
    }
    Therefore, equating these two expressions, 
    \eqnN{
    t - \hat t &= (\hat R  - R) a_i + \phi_i - \epsilon_i\\
    \therefore \norm{t - \hat t} &\leq \norm{(\hat R  - R) a_i } + \norm{\phi_i} + \norm{\epsilon_i}\\
    & \leq \norm{\hat R - R}_2 \norm{a_i} + \norm{\phi_i} + \delta_i\\
    & \leq \fronorm{\hat R - R} \norm{a_i} + \norm{\phi_i} + \delta_i\\
    & \leq \epsilon_R \norm{a_i} + \norm{\phi_i} + \delta_i
    }
    Since this holds true for all $i \in \Ical$, we take 
 the minimum of the RHS over $i$, yielding the expression in the theorem statement. 
\end{proof}

\begin{lemma}
\label{lemma:inverse_rototranslation}
For any rototranslation $\rototranslation{A}{B} \in \SE(3)$, 
    \eqnN{
    \translation{B}{A} = - \rotation{B}{A} \translation{A}{B}.
    }
\end{lemma}
\begin{proof}
Consider some vector $v\inframe{B}$. Then
\eqnN{
v\inframe{A} = \rotation{B}{A} v\inframe{B} + \translation{B}{A}\\
\therefore \translation{B}{A} = v\inframe{A} - \rotation{B}{A} v\inframe{B}
}
Furthermore, $v\inframe{B} = \rotation{A}{B} v\inframe{A} + \translation{A}{B}$. Therefore, 
\eqnN{
\translation{B}{A} &= v\inframe{A} - \rotation{B}{A} ( \rotation{A}{B} v\inframe{A} + \translation{A}{B} )\\
 &= v\inframe{A} -  \rotation{B}{A} \rotation{A}{B} v\inframe{A} - \rotation{B}{A} \translation{A}{B}\\
 &= - \rotation{B}{A} \translation{A}{B}
}
\end{proof}

\subsection{Proof of~\Cref{theorem:new_radius}}
\label{section:proof:new_radius}
\begin{proof}
Consider any point $p\inframe{B_{k+1}} \in \R^3$. When expressed in frame $B_k$, we have
\eqnN{
&\norm{p\inframe{B_k} - \hat p\inframe{B_k}}\\
&= \norm{ \left( \Brotation{k+1}{k}p\inframe{B_{k+1}} + \Btranslation{k+1}{k} \right) - \left(
\estBrotation{k+1}{k} p\inframe{B_{k+1}} + \estBtranslation{k+1}{k} \right)
}\\
~
&= \norm{ \left( \Brotation{k+1}{k}- \estBrotation{k+1}{k} \right) p\inframe{B_{k+1}} + \left(\Btranslation{k+1}{k} - \estBtranslation{k+1}{k}\right)
}\\
~
&= \norm{ \left( \Brotation{k+1}{k}- \estBrotation{k+1}{k} \right) p\inframe{B_{k+1}} + \left(-\Brotation{k+1}{k}\Btranslation{k}{k+1} +\estBrotation{k+1}{k}\estBtranslation{k}{k+1}\right)
}
}
where in the last step we used the relationship $\translation{B}{A} = - \rotation{B}{A}\translation{A}{B}$ (see~\Cref{lemma:inverse_rototranslation}). Adding and subtracting $\Brotation{k+1}{k} \estBtranslation{k}{k+1}$ inside the norm yields
\eqnN{
&\norm{p\inframe{B_k} - \hat p\inframe{B_k}}\\
&= \norm{ \left( \Brotation{k+1}{k}- \estBrotation{k+1}{k} \right) p\inframe{B_{k+1}}  + \left( -\Brotation{k+1}{k}\Btranslation{k}{k+1} +\estBrotation{k+1}{k}\estBtranslation{k}{k+1} \right) + \left(\Brotation{k+1}{k} \estBtranslation{k}{k+1} -  \Brotation{k+1}{k} \estBtranslation{k}{k+1} \right) } \\
&= \norm{ \left( \Brotation{k+1}{k}- \estBrotation{k+1}{k} \right) p\inframe{B_{k+1}} + \Brotation{k+1}{k} \left( - \Btranslation{k}{k+1} +  \estBtranslation{k}{k+1} \right) + \left( \estBrotation{k+1}{k}  -  \Brotation{k+1}{k} \right) \estBtranslation{k}{k+1} }\\
&= \norm{ \left( \Brotation{k+1}{k}- \estBrotation{k+1}{k} \right) \left( p\inframe{B_{k+1}} -  \estBtranslation{k}{k+1}   \right) + \Brotation{k+1}{k} \left( - \Btranslation{k}{k+1} +  \estBtranslation{k}{k+1} \right) }
}
Now using the triangle inequality and simplifying,
\eqnN{
&\norm{p\inframe{B_k} - \hat p\inframe{B_k}}\\
&\leq \norm{ \Brotation{k+1}{k}- \estBrotation{k+1}{k}} \norm{p\inframe{B_{k+1}} -  \estBtranslation{k}{k+1} } + \norm{\Brotation{k+1}{k} \left( - \Btranslation{k}{k+1} +  \estBtranslation{k}{k+1} \right) }\\
&= \norm{ \Brotation{k+1}{k} - \estBrotation{k+1}{k}} \norm{p\inframe{B_{k+1}} -  \estBtranslation{k}{k+1} } + \norm{ - \Btranslation{k}{k+1} +  \estBtranslation{k}{k+1}  }\\
&= \norm{ \left(\Brotation{k+1}{k} - \estBrotation{k+1}{k} \right)^T} \norm{p\inframe{B_{k+1}} -  \estBtranslation{k}{k+1} } + \norm{ - \Btranslation{k}{k+1} +  \estBtranslation{k}{k+1}  }\\
&= \norm{ \Brotation{k}{k+1} - \estBrotation{k}{k+1}} \norm{p\inframe{B_{k+1}} -  \estBtranslation{k}{k+1} } + \norm{ \Btranslation{k}{k+1} -  \estBtranslation{k}{k+1}  }\\
&\leq \epsilon_{R, k} \norm{p\inframe{B_{k+1}} -  \estBtranslation{k}{k+1} } + \epsilon_{t, k}
}

Define $\Pcal \subset \R^3$ as the set of points $p\inframe{B_k}$ that can correspond to $p\inframe{B_{k+1}}$:
\eqnN{
\Pcal = \{ p \in \R^3 : \norm{ p - \hat p\inframe{B_{k}}} \leq \Delta \},
}
where $\Delta = \epsilon_{R, k} \norm{p\inframe{B_{k+1}} -  \estBtranslation{k}{k+1} } + \epsilon_{t, k} $.
Then, 
\eqnN{
d(p\inframe{B_{k+1}}) &\geq \min_{\tilde p \in \Pcal} d(\tilde p)\\
&\geq d(\hat p\inframe{B_k}) - \Delta\\
&\geq d_M^k(\hat p\inframe{M}) - \epsilon_{R, k} \norm{p\inframe{B_{k+1}} -  \estBtranslation{k}{k+1} } - \epsilon_{t, k}
}
where $\hat p\inframe{B_k} = \estBrotation{k+1}{k} p\inframe{B_{k+1}} + \estBtranslation{k+1}{k}$, and $\hat p\inframe{M} = \estrotation{B_k}{M} \hat p\inframe{B_k} + \esttranslation{B_k}{M}$.  Therefore, we have
\eqnN{
d(p\inframe{B_{k+1}}) &\geq d_M^k(\hat p\inframe{M}) - \epsilon_R \norm{p\inframe{B_{k+1}} -  \estBtranslation{k}{k+1} } - \epsilon_t\\
&= d_M^k(\hat p\inframe{M}) - \epsilon_R \norm{\left( \estrotation{M}{B_{k+1}} \hat p\inframe{M} + \esttranslation{M}{B_{k+1}}\right) -  \estBtranslation{k}{k+1} } - \epsilon_t\\
&= d_M^{k+1}(\hat p\inframe{M})
}
by defining $d_M^{k+1}: \R^3 \to \R$ as 
\eqnN{
d_M^{k+1}(\hat p \inframe{M}) = d_M^k(\hat p\inframe{M}) - \epsilon_R \norm{ \estrotation{M}{B_{k+1}} \hat p\inframe{M} + \esttranslation{M}{B_{k+1}} -  \estBtranslation{k}{k+1} } - \epsilon_t.
}
This completes the proof. 
\end{proof}

\end{document}